%% file: main.tex
\documentclass[nonacm, screen, authorversion]{jair}
\input{header}

\setcopyright{cc}
\copyrightyear{2026}
\acmDOI{XXXXXX}

\acmVolume{X}
\acmArticle{1}
\acmMonth{1}
\acmYear{2026}

\RequirePackage[
  datamodel=acmdatamodel,
  style=acmauthoryear,
  backend=biber,
  giveninits=true,
  uniquename=init
  ]{biblatex}

\begin{document}

\title{Reactive Knowledge Representation and Asynchronous Reasoning}

\author{Simon Kohaut}
\orcid{0000-0002-0855-6316}
\affiliation{%
  \institution{Artificial Intelligence and Machine Learning Group, TU Darmstadt}
  \city{Darmstadt, Hesse}
  \country{Germany}
}
\additionalaffiliation{%
  \institution{Konrad Zuse School of Excellence in Learning and Intelligent Systems (ELIZA)}
  \city{Darmstadt, Hesse}
  \country{Germany}
}
\email{simon.kohaut@cs.tu-darmstadt.de}
\author{Benedict Flade}
\orcid{0000-0002-0636-955X}
\affiliation{%
  \institution{Honda Research Institute EU}
  \city{Offenbach am Main, Hesse}
  \country{Germany}
}
\email{benedict.flade@honda-ri.de}
\author{Julian Eggert}
\orcid{0000-0003-4437-6133}
\affiliation{%
  \institution{Honda Research Institute EU}
  \city{Offenbach am Main, Hesse}
  \country{Germany}
}
\email{julian.eggert@honda-ri.de}
\author{Kristian Kersting}
\orcid{000-0002-2873-9152}
\affiliation{%
  \institution{Artificial Intelligence and Machine Learning Group, TU Darmstadt}
  \city{Darmstadt, Hesse}
  \country{Germany}
}
\additionalaffiliation{%
  \institution{Hessian Center for Artificial Intelligence (hessian.AI)}
  \city{Darmstadt, Hesse}
  \country{Germany}
}
\additionalaffiliation{%
  \institution{Centre for Cognitive Science, TU Darmstadt}
  \city{Darmstadt, Hesse}
  \country{Germany}
}
\additionalaffiliation{%
  \institution{German Center for Artificial Intelligence (DFKI)}
  \city{Darmstadt, Hesse}
  \country{Germany}
}
\email{kristian.kersting@cs.tu-darmstadt.de}
\author{Devendra Singh Dhami}
\orcid{0000-0002-4331-7193}
\affiliation{%
  \institution{Uncertainty in Artificial Intelligence Group, TU Eindhoven}
  \city{Eindhoven, North Brabant}
  \country{Netherlands}
}
\email{d.s.dhami@tue.nl}

\renewcommand{\shortauthors}{Kohaut, Flade, Eggert, Kersting \& Dhami}

\begin{abstract}
Exact inference in complex probabilistic models often incurs prohibitive computational costs.
This challenge is particularly acute for autonomous agents in dynamic environments that require frequent, real-time belief updates.
%
Existing methods are often inefficient for ongoing reasoning, as they re-evaluate the entire model upon any change, failing to exploit that real-world information streams have heterogeneous update rates.
To address this, we approach the problem from a reactive, asynchronous, probabilistic reasoning perspective.
%
We first introduce Resin (Reactive Signal Inference), a probabilistic programming language that merges probabilistic logic with reactive programming. 
Furthermore, to provide efficient and exact semantics for Resin, we propose Reactive Circuits (RCs). 
Formulated as a meta-structure over Algebraic Circuits and asynchronous data streams, RCs are time-dynamic Directed Acyclic Graphs that autonomously adapt themselves based on the volatility of input signals.
%
In high-fidelity drone swarm simulations, our approach achieves several orders of magnitude of speedup over frequency-agnostic inference. 
We demonstrate that RCs' structural adaptations successfully capture environmental dynamics, significantly reducing latency and facilitating reactive real-time reasoning.
%
By partitioning computations based on the estimated Frequency of Change in the asynchronous inputs, large inference tasks can be decomposed into individually memoized sub-problems. 
This ensures that only the specific components of a model affected by new information are re-evaluated, drastically reducing redundant computation in streaming contexts.
\end{abstract}

\received{DD MM 2026}
\received[accepted]{DD MM 2026}

\maketitle

\clearpage
\section{Introduction}
\label{sec:introduction}
Autonomous systems, ranging from self-driving cars to robotic assistants, operate in dynamic environments where safety and reliability depend on the ability to continually reason about changing states.
Although probabilistic reasoning methods offer the necessary formal guarantees for such safety-critical applications~\cite{DBLP:books/daglib/pearl1988probabilistic,DBLP:conf/ijcai/YangIL20,DBLP:journals/jair/SkryaginODK23}, they are predominantly designed for one-shot inference.
In real-time settings, where beliefs must be continuously updated against asynchronous streams of information, these methods often incur prohibitive computational costs for every input change, regardless of its scope.

While specialized models for temporal data, such as Dynamic Bayesian Networks~\cite{murphy2002dynamic}, exist, they typically assume synchronous updates at fixed time steps and are not designed to exploit the heterogeneous, asynchronous nature of data streams found in many real-world systems.
Instead, our work builds on the key insight that information in real-world scenarios rarely changes uniformly.
For instance, an autonomous vehicle's position changes continuously, whereas traffic lights change intermittently, and road network data remains static for hours or longer.
Existing tractable models, such as Probabilistic Circuits~\cite{choi2020probabilistic}, typically fail to exploit this temporal structure, redundantly re-evaluating stable parts of the model alongside volatile ones.

To address this inefficiency, we propose a novel approach that dynamically adapts the inference process to the rate of information change, treating the data channels between soft- and hardware components as first-class citizens.
We introduce \textbf{Reactive Circuits (RCs)}, a time-dynamic inference framework built on top of Algebraic Circuits~\cite{DBLP:journals/japll/KimmigBR17} that serves as an efficient, exact, and asynchronous inference engine.
RCs automatically track the Frequency of Change (FoC) of their input signals and partition their computational structure accordingly.
By grouping inputs by volatility and memoizing intermediate results, RCs eliminate redundant calculations, re-evaluating only the model components affected by new information.
This dynamic restructuring endows the circuits with computational \textit{plasticity}, allowing them to adapt their complexity to mirror the environment's dynamics.
The resulting reactive evaluation can yield speedups of several orders of magnitude, enabling continuous exact inference in time-critical applications.

As a high-level, declarative modeling approach of such systems, we present \textbf{Resin (Reactive Signal Inference)}.
Resin is a first asynchronous probabilistic programming language that bridges the gap between the expressivity of Answer Set Programming (ASP) and the practical requirements of robotics, such as handling asynchronous data streams from sensors or deep learning models.
Resin allows users to declaratively specify source signals for exchanging information with separate processes as well as logical target signals, which are then compiled into RCs for efficient, continual evaluation.
In summary, our contributions are two-fold: 
\begin{description}
    \item[\textbf{Resin}] A high-level probabilistic logic language designed for modeling continual inference tasks in distributed systems. 
    It bridges the gap between the expressivity of Answer Set Programming and the requirements of reactive systems, allowing for the joint specification of asynchronous data flows and logic constraints.
    \item[\textbf{Reactive Circuits}] A novel, adaptive inference structure for asynchronous reasoning that generalizes over any commutative semiring. 
    By dynamically restructuring themselves based on signal frequency, RCs minimize redundant operations through strategic memoization, enabling real-time exact inference.
\end{description}
Together, these components enable the development of reactive, trustworthy autonomous agents.
Complementary to these contributions, we provide an open source implementation available at~\href{https://github.com/simon-kohaut/resin}{github.com/simon-kohaut/resin}.

The remainder of this paper is structured as follows.
Sec.~\ref{sec:resin} formalizes Resin, detailing its syntax for probabilistic modeling, its data model for handling asynchronous signals, and the compilation process into RCs.
Sec.~\ref{sec:reactive_circuits} introduces Reactive Circuits, establishing the algebraic foundations and defining the operational semantics for frequency-based adaptation and reactive evaluation.
We then present our experimental evaluation in Sec.~\ref{sec:experiments}, where we investigate tracking accuracy, circuit plasticity, and performance gains in both synthetic benchmarks and high-fidelity drone simulations.
Finally, Sec.~\ref{sec:related_work} discusses related work in neuro-symbolic AI and tractable inference before concluding with a summary and outlook in Sec.~\ref{sec:conclusion}.

\input{figures/architecture}

\section{The Resin Programming Language}
\label{sec:resin}

Probabilistic logic is an excellent tool for modeling constrained decisions and motion planning.
To this end, Resin draws inspiration from works such as ProbLog~\cite{DBLP:conf/ijcai/RaedtKT07,DBLP:conf/nips/ManhaeveDKDR18}, NeurASP~\cite{DBLP:conf/ijcai/YangIL20}, and SLASH~\cite{DBLP:conf/kr/SkryaginSODK22}, which blend probabilistic inference with first-order logic and integrate ASP, respectively. 
Resin extends these concepts to integrate inter-process communication with probabilistic first-order logic, targeting settings where asynchronously updated parameters in an uncertain domain knowledge are critical for ensuring ongoing safety constraints.

While Sec.~\ref{sec:reactive_circuits} discusses RCs as an online, self-adapting reasoning framework that provides efficient semantics for Resin, here we discuss the Resin programming language itself as a high-level reactive knowledge representation.
Hence, the following discusses the structure and internals of Resin programs, their data model and data exchange via Data Distribution Service (DDS) channels~\cite{DBLP:conf/icdcsw/Pardo-Castellote03}, and how we obtain initial RCs from a Resin program and target.

\subsection{Asynchronous Probabilistic Logic}
Resin couples inter-process communication and first-order logic into a single framework suitable for modeling continual inference processes and their embedding into networks of (neural) sensors and arbitrary computational models.
Asynchronous data channels carry Resin's parameter updates and inference results during runtime, i.e., continually distributing input and output data.
\begin{definition}
    A \textbf{signal} $s$ is defined by the tuple $(a, c, T)$, where $a$ is an atom, $c$ is a communication channel identifier, and $T$ is a data type.
    At time $t \in \mathbb{R}_{\geq 0}$, $s$ is associated with a value $s(t) \in \mathbb{R}$ which updates at a rate $\lambda_s(t)$.
    That is, we assume a non-stationary Poisson process dictating the arrival of updates to $s(t)$.
\end{definition}%
One can further distinguish signals according to their direction within a network of data distributors.
Signals are designated as \textbf{sources} or \textbf{targets} of a Resin program.
While source signals capture incoming data with unknown temporal behavior, target signals are computed using RCs and shared across the network.

Let us now take a look at how to build an entire Resin program.
\begin{definition}
    A \textbf{Resin program} is a first-order theory $\mathcal{T}$ that exchanges data via source signals $s_n \in \mathcal{S}$, $|S| = N$ and target signals $d_m \in \mathcal{D}$, $|D| = M$.
    Hereby, theory $\mathcal{T}$ is defined over a set $\mathcal{K} = \mathcal{S} \cup \mathcal{D} \cup \mathcal{A}$ of atoms.
    In addition to the signals, $\mathcal{A}$ captures background knowledge, connecting sources to targets.
\end{definition}%
The syntax of a Resin program is provided in Backus-Naur form in Appendix~\ref{app:grammar}, showing how sources, targets, and first-order logic concepts are formalized.
Fig.~\ref{fig:architecture} provides an overview of the Asynchronous Reasoning architecture enabled by Resin in tandem with RCs, illustrating the process of translating a Resin program into RCs, which in turn provide continual updates to the target signals.

\subsection{Data Model}
Resin's data model is designed to bridge the gap between continuous, asynchronous data streams and discrete probabilistic logic.
Signals in Resin are typed, ensuring that data from external processes is correctly interpreted within the logical framework.
We distinguish four core data types:
\begin{itemize}
    \item \texttt{Boolean}: Represents binary states ($\top, \bot$), typically derived from logical sensors or switches.
    \item \texttt{Number}: Represents continuous values $x \in \mathbb{R}$, such as sensor readings (e.g., velocity, altitude).
    \item \texttt{Probability}: Represents a degree of belief $P(x) \in [0, 1]$, often the output of an upstream classifier.
    \item \texttt{Density}: Represents a probability density function $p(x)$ over a continuous variable, capturing uncertainty in measurements like location or object distance.
\end{itemize}

While Resin's syntax is grounded in untyped first-order logic, these types dictate how signal values are coerced into the weights of the underlying reasoning task.
Depending on the reasoning task at hand, different semirings and coercions may be employed.
For example, in the context of Weighted Model Counting (WMC), where the semiring of interest is $([0, 1], +, \times, 0, 1)$, this coercion is handled as follows:
\begin{description}
    \item[\textbf{Direct Mapping}] Signals of type \texttt{Probability} are directly assigned to the corresponding atom.
    \item[\textbf{Determinism}] \texttt{Boolean} signals are assigned a probability of $1.0$ if true and $0.0$ otherwise.
    \item[\textbf{Integration}] Comparison operators are used to bridge the gap between continuous values and discrete logic. 
    For \texttt{Number} signals, comparisons (e.g., \texttt{speed < 20.0}) are treated analogously to Booleans. 
    Similarly, a \texttt{Density} involved in a comparison (e.g., \texttt{clearance > 10.0}) is coerced into a probability via its Cumulative Distribution Function (CDF).
\end{description}

This type system not only facilitates the integration of heterogeneous data sources but also plays a role in Resin's reactive runtime.
That is, the chosen domain of the source signals has implications on the volatility of the signals' runtime behaviour (and by extension, the structure of the Reactive Circuits).
For instance, while a \texttt{Density} signal representing a drone's position may change continuously, a derived \texttt{Probability} of it being in an unsafe region may remain constant for long periods (see Sec.~\ref{sec:experiments}).

\subsection{Compiling Resin}
Resin programs are translated into Answer Set Programs (ASP).
Specifically, sources are mapped to choices and rules are syntactically translated (e.g., substituting an \textit{and} in Resin with a comma in ASP).
For each target signal, a constrained ASP is generated that ensures the respective target holds in all answer sets.
We employ the off-the-shelf ASP solver Clingo~\cite{DBLP:journals/tplp/GebserKKS19} to exhaustively search for the stable models of these programs.

Let $\mathcal{J}_m$ be the set of stable models yielded by the ASP solver given the Resin program's theory $\mathcal{T}$ constrained on the respective ground atom of target signal $d_m \in \mathcal{D}$.
The probability of a stable model is computed as the product of the probabilities of the probabilistic facts as they appear in the model~\cite{DBLP:conf/kr/SkryaginSODK22}.
Consequently, the probability of a target signal is obtained as the sum of the probabilities of its stable models.
\begin{align}
   P(d_m) = \sum\nolimits_{j \in \mathcal{J}_m} \prod\nolimits_{s \in \mathcal{S}} P(s = j(s))
    \label{eq:wmc}
\end{align}
This problem, known as Weighted Model Counting (WMC), is traditionally subject to additional knowledge compilation steps.
That is, although the formula is correct, one can often represent the problem in much smaller forms than Eq.~\ref{eq:wmc}.
Going forward, we will assume this unrefined formulation of WMC without loss of generality when discussing the semantics of Resin in the form of Reactive Circuits and their online self-adjustments.

For each target $d_m$, the respective formula is provided to the RC-based online inference framework as laid out in Sec.~\ref{sec:reactive_circuits}.
More specifically, each source and target signal is connected to a DDS framework to exchange data with the compiled and adapted RCs.

We have focused on probabilistic inference because it highlights Resin's utility for continual safety and legal checks in robotics tasks.
Of course, parameter learning and other inference classes are also relevant, e.g., as supported by languages such as ProbLog~\cite{DBLP:journals/tplp/FierensBRSGTJR15}, DeepProbLog~\cite{DBLP:conf/nips/ManhaeveDKDR18}, or SLASH~\cite{DBLP:conf/kr/SkryaginSODK22}.
For this reason, we demonstrate in Sec.~\ref{sec:reactive_circuits} how RCs generalize to any commutative semiring, i.e., how they support tackling other tasks, such as gradients or most probable explanations.

\section{Reactive Circuits}
\label{sec:reactive_circuits}

Our contribution of asynchronous reasoning is built on the properties of the algebraic structures that underlie common reasoning tasks, i.e., sum-product type problems such as WMC.
Hence, we begin this section by defining the abstract algebra needed to support asynchronous reasoning with Reactive Circuits.
First, as a basic building block, let us consider the monoid.
\begin{definition}
    A \textbf{monoid} is an algebraic structure $(R, \circ, e^\circ)$ consisting of a set $R$ and the binary, associative operator $\circ$ with neutral element $e^\circ$.
    For $a, b, c \in R$, the equations $a \circ (b \circ c) = (a \circ b) \circ c$ and $a \circ e^\circ = e^\circ \circ a = a$ hold.
    Further, a commutative monoid supports $a \circ b = b \circ a$.
\end{definition}%
While monoids are insufficient for solving tasks such as WMC, which require two binary operators, we can utilize them in the following definition.
\begin{definition}
    A \textbf{commutative semiring} is an algebraic structure $(R, \oplus, \otimes, e^\oplus, e^\otimes)$ such that $(R, \oplus, e^\oplus)$ and $(R, \otimes, e^\otimes)$ each are commutative monoids.
    Additionally, they require that $\otimes$ left and right distributes over $\oplus$, i.e., for $a, b, c \in R$, it holds that $a \otimes (b \oplus c) = (a \otimes b) \oplus (a \otimes c) = (b \oplus c) \otimes a$.
\end{definition}%
As an example, with $R = [0, 1]$, classic addition and multiplication, an algebraic circuit can express Equation~\ref{eq:wmc} and thereby facilitate exact inference of a Resin \texttt{Probability} target.
Finally, as our further discussions will center around computation trees, we must consider the algebraic circuit as a graphical representation of arbitrary expressions within the underlying algebraic structure.
\begin{definition}
    An \textbf{Algebraic Circuit} is a Directed Acyclic Graph (DAG) over elements and operators of a commutative semiring.
    The circuit's nodes are either gates, i.e., $n$-ary applications of $\oplus$ or $\otimes$ on their respective child nodes, or leaves, which take on values from the underlying structure's domain $R$.
\end{definition}%
We now define the operational semantics of asynchronous reasoning.
To this end, we introduce a novel, adaptive inference structure we call Reactive Circuits (RC).
Since we aim to continually update inference results in large problems, our approach is to adapt the inference process to minimize the time required to compute the target signal.
We define Reactive Circuits as follows.
\begin{definition}
    A \textbf{Reactive Circuit} $\mathcal{RC}$ is a tuple $(\mathcal{F}, \mathcal{S}, \mathcal{E}, \mathcal{M})$ of nodes $\mathcal{F}$ and $\mathcal{S}$, edges $\mathcal{E}$ and memorized values $\mathcal{M}$.
    Each node $f_i \in \mathcal{F}$ is associated with an AC, each node $s_n \in \mathcal{S}$ is a source signal, and each edge in $\mathcal{E}$ is a directed edge connecting formula and source nodes as a DAG.
    Furthermore, each formula $f_i$ is associated with a memory $m_i$ that stores its output.
\end{definition}%
As shown in Fig.~\ref{fig:safety_rc}, RCs are an extension of Algebraic Circuits, where an original formula is distributed into a DAG of sub-formulas, maintaining a memorized intermediate result for each.
The memory $m_i$ for each node $f_i$ is given by evaluating the formula $f_i$ over its children $C(f_i)$:
\begin{align}
    \label{eq:rc_value}
    m_i(t) = f_i \left( \{ s_j(t) \mid s_j \in C(f_i) \}, \{ m_k \mid m_k \in C(f_i) \} \right).
\end{align}
Hence, the resulting RC represents a multilinear polynomial over the weights of the stable models $\mathcal{J}$.

Beyond probabilistic inference for, e.g., Resin programs, the operations on the RC's DAG we will show in the following are independent of the choice of commutative semiring.
For a more in-depth discussion on applying commutative semirings across a multitude of reasoning tasks, please refer to~\cite{DBLP:journals/japll/KimmigBR17}.

\input{figures/safety_rc}

While our presentation revolves around formulas in Disjunctive Normal Form, i.e., with a single sum over product gates in the case of WMC, they can similarly encapsulate more refined formulas, e.g., optimized for size, without losing generality.
An example program, RC, and internal Algebraic Circuits are illustrated in Fig.~\ref{fig:safety_rc}.

\subsection{Frequencies of Change}
The core mechanism of RCs is their ability to adapt to the rate of change of their inputs. 
To formalize this, we must first define what constitutes a "meaningful" change that warrants a re-computation.
\begin{definition}
    A Predicate of Change $\phi$ is a binary indicator function that determines if a value $v$ constitutes a meaningful update to the value held by signal $s_n$.
\end{definition}%
For example, for a numerical signal, a practical predicate might be to check if the change exceeds a certain threshold, i.e., $\phi(s_n(t), v) = 1$ if $|s_n(t) - v| > \epsilon, \epsilon \in \mathbb{R}_{>0}$. This allows the system to ignore minor sensor noise or insignificant fluctuations. 
The choice of $\phi$ introduces a critical trade-off: 
a lenient predicate (larger $\epsilon$) reduces the computational load but may compromise the exactness of the inference, while a strict predicate (smaller $\epsilon$) ensures higher fidelity at the cost of more frequent updates.

Given a predicate $\phi$, we can now define the rate of meaningful updates.
\begin{definition}
    The Frequency of Change (FoC) $\focn \in \mathbb{R}_{\geq 0}$ of a signal $s_n$ is the expected rate of incoming messages at time $t$ that satisfy the predicate $\phi$. 
    The FoC is measured in Hz and is generally non-stationary, reflecting the environment's dynamics.
\end{definition}%
The FoC is essentially a filtered version of the raw message arrival rate $\lambda_{s_n}(t)$. 
Since $\phi$ only filters messages and never creates new ones, it trivially holds that $\focn \leq \lambda_{s_n}(t)$ at all times. 
In the next section, we discuss how these non-stationary FoCs are estimated online to inform the circuit's adaptation strategy.

\input{figures/adaptation}

\input{figures/operations}

\subsection{Online Circuit Adaptation}
The key to RC efficiency is dynamically adapting the circuit's structure to match the time-varying Frequency of Change (FoC) of its input signals. 
This adaptation process involves two main steps: 
(1) estimating the FoC for each signal online, and (2) restructuring the circuit based on these estimates.

For online FoC estimation, we model the arrival of meaningful updates for each signal as a noisy process. 
We track the time between updates, $\Delta t_{\phi, n} = 1 / \focn$, using a Kalman filter~\cite{MANUAL:kalman1960new}. 
Specifically, we track the state $\mathbf{x} = (\Delta t_{\phi, n}, \dot{\Delta t_{\phi, n}})^T$, representing the inter-arrival time and its rate of change. 
We assume a linear, constant velocity model for state evolution and a direct observation of the inter-arrival time, with system dynamics and measurement models defined as:
\begin{equation*}
    \label{eq:f_and_h}
    \mathbf{F} = \left[ \begin{array}{cc}
         1 & \delta t \\
         0 & 1
    \end{array} \right] \quad \text{and} \quad
    \mathbf{H} = \left[ \begin{array}{cc}
         1 & 0
    \end{array} \right].
\end{equation*}
This standard filter configuration, assuming Gaussian noise, allows us to robustly estimate the underlying FoC for each source signal from the noisy measurements of time between changes provided by a monotonic clock.

With the online FoC estimates available, we partition the signals into discrete frequency bands. 
This is achieved by clustering them based on their FoC values. 
For a given partition width $h \in \mathbb{R}^+$, each signal $s_n$ is assigned to a cluster $c_n(t)$ such that:
$$\forall k \in \mathbb{N}_0 : kh \leq \focn < (k + 1)h \iff c_n(t) = k.$$
The partition width $h$ is a crucial hyper-parameter that, as we show in Sec.~\ref{sec:experiments}, governs the trade-off between circuit size and computational speedup.

As a signal's estimated FoC evolves over time, it may cross a partition boundary, triggering a structural adaptation of the RC. 
This restructuring is performed by the \textit{lift} and \textit{drop} operations (see Algorithms~\ref{alg:lift} and \ref{alg:drop} in Appendix~\ref{app:algorithms}), which move sets of signals and the associated computations between different formula nodes of the RC's DAG, as illustrated in Fig.~\ref{fig:adaptation}. 
Most critically, these operations are guaranteed to preserve the circuit's semantics, i.e., the value of the target signal computed depending on source signal values.

\begin{lemma}
Let $root(\mathcal{RC})(t)$ be the value of the root of the RC's DAG at time $t$.
Then, the circuit adaptations \textit{lift} and \textit{drop} conserve its value, i.e., $\forall s \in \mathcal{S}: root(lift(\mathcal{RC}, \{s\}))(t) = root(drop(\mathcal{RC}, \{s\}))(t) = root(\mathcal{RC})(t)$.
\label{lemma:preservation}
\end{lemma}

\begin{proof}
The correctness of \textit{lift} and \textit{drop} stems from the distributive property of $\otimes$ over $\oplus$ in the underlying commutative semiring (Definition 3.2). 
For instance, the identity $a \otimes (b \oplus c) = (a \otimes b) \oplus (a \otimes c)$ corresponds to a \textit{drop} of variable $a$ when moving from the left-hand side's structure to the right's, or a \textit{lift} of $a$ when moving from right to left. 
The algorithms apply this principle by restructuring the RC, both by connecting and disconnecting signal nodes in the RC itself, and by appropriately altering the individual AC's associated with the respective formula nodes in accordance with Definition 3.2.
\end{proof}

\subsection{Reactive Circuit Evaluation}
Reactive evaluation avoids redundant processing by re-computing only those parts of the circuit affected by an incoming signal update. 
When a source signal $s_n$ receives a meaningful update (i.e., one that satisfies $\phi$), the memoized values of its dependent nodes are no longer guaranteed to be correct. 
We refer to these nodes as \textit{invalidated}. 
The set of invalidated nodes, denoted $Dep(s_n)$, comprises all formula nodes whose computation directly or indirectly depends on $s_n$ and all their ancestors. 

\begin{lemma}
\label{lemma:dependency}
Let $Q$ be a sorted list of indices of invalidated formulas, i.e., formulas which have a descendant $s_n$ for which a new value satisfying $\phi$ was received.
For any $i, j \in Q$, if formula $f_j$ is an ancestor of formula $f_i$, then $j < i$.
\end{lemma}

\begin{proof}
Formula node indices are assigned via a topological sort of the RC's DAG (e.g., using a depth-first traversal). 
Thus, when processing invalidated nodes in decreasing order of their indices, a node is always evaluated after its dependencies have been updated.
\end{proof}

Hence, updating an RC involves maintaining a sorted queue $Q$ containing the indices of all invalidated formulas in $Dep(s_n)$.
For each index $i$ in $Q$, the corresponding formula $f_i$ is re-evaluated using the memoized values of its children, and its own memoized value $m_i$ is updated. 
Each time a formula and its associated memory are updated, the corresponding index is removed from $Q$.
Once the queue is empty, the RC is again in a consistent state. 
This reactive evaluation provides significant computational savings, as formalized below.

\begin{theorem}
\label{theorem:validation}
Let $\omega_i$ be the number of operations to evaluate formula node $f_i$, and $\Omega = \sum_i \omega_i$ be the total operations for a full RC evaluation.
We then obtain the operation rate for a naive evaluation (re-computing the entire RC on any change) $\rho_\text{MAX}(t)$, the rate for a reactive evaluation $\rho_\text{RC}(t)$, and the corresponding efficiency gain $\rho_\text{GAIN}(t)$.
\begin{align*}
    \rho_\text{MAX}(t) &=  \sum\nolimits_n \lambda_{\phi, n}(t) \ \Omega \tag{Full} \\
    \rho_\text{RC}(t) &= \sum\nolimits_n \lambda_{\phi, n}(t) \sum\nolimits_{i \in Dep(s_n)} \omega_i \tag{Reactive} \\
    \rho_\text{GAIN}(t) &= \rho_\text{MAX}(t) / \rho_\text{RC}(t). \tag{Gain}
\end{align*}
\end{theorem}

\begin{proof}
Correctness is established by ensuring that all invalidated nodes are re-evaluated (completeness) and that they are evaluated in an order that respects data dependencies (soundness). 
Lemma~\ref{lemma:dependency} guarantees a sound topological evaluation order. 
Completeness is ensured by construction, as $Dep(s_n)$ includes all nodes affected by a change in $s_n$. 
The semantic preservation of the circuit's function across adaptations is given by Lemma~\ref{lemma:preservation}.
The sum of the frequencies $\lambda_{\phi, n}$ follows from the superposition principle of independent Poisson processes.

To prove the performance gain, we compare the terms for $\rho_\text{RC}(t)$ and $\rho_\text{MAX}(t)$. 
For any signal $s_n$, the set of its dependent nodes $Dep(s_n)$ is a subset of the set of all formula nodes. 
As the operational cost $\omega_i$ is non-negative, the cost of a partial update is bounded by the cost of a full update: 
$$\sum_{i \in Dep(s_n)} \omega_i \leq \sum_j \omega_j = \Omega.$$
Consequently, $\rho_\text{RC}(t) \leq \rho_\text{MAX}(t)$, hence $\rho_\text{GAIN}(t) \geq 1$.
\end{proof} 

\subsection{Example}

Let us consider the case of Weighted Model Counting for exact inference in the following Resin program, i.e., $\otimes$ and $\oplus$ are classic multiplication and summation over values in the $[0, 1]$ interval.

\begin{listing}[h]
    \caption{
        An example Resin program. 
    }%
    \label{listing:example_resin_program}
    \centering
    \begin{minted}
    [
        frame=none,
        autogobble,
        fontsize=\footnotesize,
        xleftmargin=20pt,
        linenos
    ]{python}
    a <- source("/a", Probability).
    b <- source("/b", Probability).
    c <- source("/c", Probability).

    d if a and b and not c.
    d if not a and b and c.
    
    d -> target("/d").
    \end{minted}
\end{listing}%

The Resin program in Listing~\ref{listing:example_resin_program} is compiled into an Answer Set Program (ASP) to find its stable models. 
Source signals `a`, `b`, and `c` are treated as choices, while the target `d` acts as a constraint. 
The resulting ASP is shown in Listing~\ref{listing:example_asp_program}.
\begin{listing}[h]
    \caption{
        The result of compiling Listing~\ref{listing:example_resin_program} for target $d$. 
    }%
    \label{listing:example_asp_program}
    \centering
    \begin{minted}
    [
        frame=none,
        autogobble,
        fontsize=\footnotesize,
        xleftmargin=20pt,
        linenos
    ]{prolog}
    0{a; b; c}3.
    
    d :- not a, b, c.
    d :- a, b, not c.
    
    :- not d.
    \end{minted}
\end{listing}%

Solving this ASP yields two answer sets: $j_1 = \{a, b, \neg c\}$ and $j_2 = \{\neg a, b, c\}$.
The probability of the target is the sum of the probabilities of its models: 
$$P(d) = P(j_1) + P(j_2) = P(a) \cdot P(b) \cdot P(\neg c) + P(\neg a) \cdot P(b) \cdot P(c).$$
At time $t_0$, the RC is set up to compute $P(d)$ according to this formula in node $f_0$.
Assume that at $t_1$, the probability of $a$ changes at a frequency of $5$~Hz and the rest at $1$~Hz.
Meaning, the RC can apply \textit{drop}$(\mathcal{RC}, \{b, c\})$ to facilitate appropriate internal memorization.
\begin{align*}
    t_0:& \qquad P(d) = m_0 \\
    & \qquad m_0 = f_0(a, b, c) = P(a) \cdot P(b) \cdot P(\neg c) + P(\neg a) \cdot P(b) \cdot P(c) \\
    t_1:& \qquad P(d) = m_0 \\
    & \qquad m_0 = f_0(a, m_1, m_2) = P(a) \cdot m_1 + P(\neg a) \cdot m_2 \\
    & \qquad m_1 = f_1(b, c) = P(b) \cdot P(\neg c) \\
    & \qquad m_2 = f_2(b, c) = P(b) \cdot P(c)
\end{align*}%
The original formula had $5$~operations (number of $\cdot$ and $+$) to perform regardless of which input changed.
Instead, after $t_1$, $3$~operations are performed at $5$~Hz when the probability of $a$ changes.
Changes to the probabilities of $b$ and $c$ still require the original $5$~operations, but due to the lower frequency, only at $1$~Hz.

\input{figures/time}

\section{Experiments}
\label{sec:experiments}
Since Resin and RCs are tailored for real-time applications, e.g., ensuring safety constraints for multi-agent systems, our experiments investigate the tracking accuracy, time, and space properties when solving reasoning tasks, such as WMC, continuously. 
Specifically, we answer the following questions through synthetic and high-fidelity simulation experiments:
\begin{itemize}
    \item[] \textbf{(Q1)} Can we accurately track the FoC and obtain a correct source signal partitioning?
    \item[] \textbf{(Q2)} How does partitioning impact reactive computation times and circuit size?
    \item[] \textbf{(Q3)} How does the frequency adaptation of RCs capture dynamic systems?
    \item[] \textbf{(Q4)} What are the gains in computation times for structural adaptation and reactiveness, respectively?
\end{itemize}
To answer these questions, we are employing data distribution based on ROS2 and Fast DDS~\cite{ros2}.
All experiments are run on an Intel i7-9700K desktop CPU and implemented in Rust~\cite{matsakis2014rust}, providing a reliable, thread-safe setup for asynchronous data exchange and reasoning with RCs.
We have aimed to choose optimal compilation settings for both baseline and adapted RC runs.

\subsection{Tracking Frequencies of Change}
To address \textbf{Q1}, we employ a Kalman filter as described in Sec.~\ref{sec:reactive_circuits}.
Fig.~\ref{fig:foc_est} shows how the filter accurately tracks a non-stationary FoC, jumping every $20$ measurements to a new, uniformly distributed frequency.
Furthermore, we present the mean absolute error (MAE) of the partitioning, i.e., the number of bins by which the estimated FoC deviates from its true position.
This demonstrates that the partitioning granularity dictates the resulting MAE: a small bin size corresponds to more mispartitioning, and vice versa.
Sensible filter tuning is required to reduce overshoot and ensure stability.

\subsection{Frequency Partitioning of Reactive Circuits}
We resolve \textbf{Q2} by (i) drawing synthetic FoCs from a set of Gaussian distributions and (ii) breaking down results for a range of partitioning parameters, i.e., the partition width $h$ as described in Sec.~\ref{sec:reactive_circuits}.
To demonstrate RC properties in isolation, we solve a generated WMC problem over $10^5$ models, each with a combination of $10^3$ sources uniformly chosen from a set of $2\times10^3$ overall signals.
Here, we provide ground-truth FoCs and partitioning to the circuit adaptation routine.
Furthermore, the reactive execution scheme is employed, i.e., memorization between signal updates is utilized to reduce the number of operations per second, as stated in $\rho_{RC}$.

Fig.~\ref{fig:rc_adapt} shows the resulting speedup, added memory requirements, and number of layers across settings.
More fine-granular frequency partitions (i) increase the memory footprint and depth of the circuit and (ii) reduce computation times when compared to baseline (non-adapted) performance.
Hence, the adaptations have a direct trade-off between size and performance, with additional memory allocation corresponding to faster inference.

\input{figures/simulation}
\input{listings/flight_conditions}
\input{figures/airsim_time}

\subsection{Rapid and Reactive Safety Assertions in Advanced Air Mobility}
To answer \textbf{Q3}, we employ AirSim~\cite{DBLP:conf/fsr/ShahDLK17} as presented in Fig.~\ref{fig:simulation}, a high-fidelity physics simulation for aerial and ground-based vehicles.
In our setup, seven Unmanned Aircraft Systems (UAS) undertake randomized journeys across a suburban environment.
Each trajectory is obtained via simulated Inertial Navigation Systems, yielding locations over time as 2D Gaussian estimates.
Similarly, the pairwise distance between all drones is computed and shared across the DDS.
The employed Resin (see Listing~\ref{listing:flight}) receives each pairwise \texttt{distance} and defines a target signal \texttt{unsafe}, which is satisfied under any combination of drones flying closer than $25$~m.

Answering \textbf{Q3}, we show that RCs successfully capture and adapt to the application's dynamics. 
Fig.~\ref{fig:leafs} shows the result: as the drones maneuver, the estimated FoCs of the various distance signals fluctuate (visualized in Fig.~\ref{fig:simulation} (c)). 
In response, the RC dynamically redistributes the corresponding signal nodes between two layers (groups of nodes at the same distance from the root), effectively segregating high-frequency (dynamic) signals from low-frequency (quasi-static) ones. 
Even with this simple two-layer partitioning, the experiment clearly demonstrates the circuit's plasticity and its ability to mirror the system's changing temporal characteristics.

We address \textbf{Q4} within the same simulation setup we employ to answer \textbf{Q3}.
The resulting WMC problem of the target signal involves $42$~source signal value streams (each datastream and its negation) and, due to any combination of nearby flying drones constituting an unsafe state, $2^{21} - 1$~models.
Fig.~\ref{fig:airsim_time} presents our ablation study of the runtime for this WMC problem, comparing three approaches.
The 'flat' baseline evaluates the original sum-product problem, representing a naive full re-computation ($\rho_\text{MAX}$).
The 'adapted' curve also performs a full re-computation on each update, but on a circuit that has been structurally optimized via our \textit{lift} and \textit{drop} operations; the speedup here demonstrates the benefit of frequency-guided knowledge compilation alone.
Finally, the 'reactive' curve shows the full potential of RCs, combining structural adaptation with partial re-computation of only the invalidated circuit parts.
This reactive evaluation, corresponding to $\rho_\text{RC}$ from Theorem~\ref{theorem:validation}, yields speedups of several orders of magnitude over the baseline by avoiding redundant computations.

\section{Related Work}
\label{sec:related_work}

\subsection{Neuro-symbolic AI}
The field of logic programming, initiated with the development of Prolog~\cite{DBLP:conf/hopl/ColmerauerR93}, has given rise to powerful paradigms for knowledge representation and reasoning.
Among these, Answer Set Programming (ASP) has become a dominant approach for declarative problem solving, based on the stable model semantics~\cite{DBLP:conf/iclp/GelfondL88}.
In parallel, to address the inherent uncertainties found in real-world applications such as robotics~\cite{DBLP:journals/cacm/Thrun02}, Probabilistic Logic Programming emerged as a key paradigm~\cite{DBLP:series/synthesis/2016Raedt}.
This paradigm, with systems like Bayesian Logic Programs~\cite{DBLP:conf/ilp/KerstingR08} and ProbLog~\cite{DBLP:conf/ijcai/RaedtKT07,DBLP:journals/tplp/FierensBRSGTJR15}, enables probabilistic inference over discrete and later extended to hybrid domains~\cite{DBLP:journals/ml/NittiLR16,DBLP:journals/jair/KumarKR23}.
The integration of deep learning has further expanded this field, leading to neuro-symbolic systems such as DeepProbLog~\cite{DBLP:conf/nips/ManhaeveDKDR18} and SLASH~\cite{DBLP:conf/kr/SkryaginSODK22} that support end-to-end learning by combining neural networks with formal logic.
These approaches are vital for autonomous agents that require safe and reliable decision-making and action planning~\cite{DBLP:conf/aaai/van2010dtproblog,DBLP:conf/ilp/MoldovanORMS11,DBLP:journals/arobots/AntanasMNFKSR19}.

Our work builds upon this foundation by introducing Resin, a probabilistic logic programming language designed for distributed systems with asynchronous data streams, and Reactive Circuits, which provide an adaptive execution model to optimize continual inference.

\subsection{Circuits and Tractable Inference}
A key approach to achieving tractable inference is to compile probabilistic models into computation graphs, often called circuits.
Probabilistic Circuits are a prominent class of such models, offering inference times linear in the circuit size~\cite{choi2020probabilistic,DBLP:conf/iccvw/PoonD11,DBLP:journals/pami/Sanchez-CaucePD22}.
This family includes Arithmetic Circuits, which are widely used for probabilistic inference~\cite{DBLP:journals/jacm/Darwiche03}, and their generalization, Algebraic Circuits, which support a broader range of tasks by operating over any commutative semiring~\cite{DBLP:journals/japll/KimmigBR17,DBLP:conf/ecai/DerkinderenR20}.
The efficiency of these circuits heavily depends on their structure, and significant research has focused on techniques for structural optimizations~\cite{DBLP:journals/jair/darwiche2002knowledge,DBLP:conf/ecai/Darwiche04,DBLP:conf/ai/MuiseMBH12,DBLP:conf/nips/SumerAIM07,DBLP:conf/icml/PeharzLVS00BKG20}.

Our work extends this paradigm by introducing Reactive Circuits (RC) that build on Algebraic Circuits as a highly generalized reasoning foundation.
Instead of creating a single, static circuit, RCs perform a continuous, frequency-guided online knowledge compilation, dynamically restructuring the computation to exploit the heterogeneous update rates of input signals.

\subsection{Reactive and Event-driven Systems} 
Our work draws inspiration from reactive and event-driven systems.
In such systems, e.g., Functional Reactive Programming~\cite{DBLP:conf/pldi/WanH00,DBLP:conf/icfp/ElliottH97}, data flows propagate changes automatically through a computation graph, with updated inputs leading to updated outputs~\cite{DBLP:conf/aosd/SalvaneschiHM14} and learning~\cite{DBLP:journals/automatica/SolowjowT20}.
We extend this paradigm to runtime-efficient weighted first-order logic in cyber-physical systems~\cite{DBLP:journals/tii/KangKS12}, presenting the first syntax and semantics for reactive and asynchronous reasoning.
To this end, our approach shares principles with event-triggered systems that compartmentalize and memoize computations~\cite{DBLP:journals/jossw/BagaevPV23,DBLP:conf/cloud/BhatotiaWRAP11}.

In contrast to searching for a single computation model at the beginning or considering only the original computation graph for minimization, RCs observe the Frequency of Change (FoC) in their inputs to trigger incremental, reversible structural refinements.
This dynamic, event-driven adaptation of the computation graph is what distinguishes RCs from traditional, static models such as Algebraic Circuits and makes them highly practical for continual reasoning.

\section{Conclusion}
\label{sec:conclusion}

In this work, we address the challenge of the high cost of continual inference for autonomous agents in dynamic environments. 
We introduce a novel reactive knowledge representation framework comprising two core contributions: \textbf{Resin}, an asynchronous probabilistic programming language, and \textbf{Reactive Circuits (RCs)}, an adaptive inference structure. 
Resin provides a high-level, declarative language for modeling continuous reasoning over asynchronous data streams, bridging probabilistic logic and reactive programming. 
RCs, in turn, provide an efficient and exact execution model by dynamically restructuring the underlying algebraic computations to mirror the volatility of the input signals.

The efficiency of RCs stems from their ability to partition computations based on the estimated Frequency of Change of input signals. 
By memoizing stable intermediate results and only re-evaluating parts of the model affected by new information, our approach drastically reduces redundant calculations. 
Our experiments, conducted in a high-fidelity drone swarm simulation, show the effectiveness of this reactive evaluation, achieving speedups of several orders of magnitude over traditional, non-adaptive inference. 
This gain highlights the power of aligning computational effort with the environment's temporal dynamics, though it introduces a trade-off between speed and memory footprint.

Future work may focus on enhancing the adaptability and scope of Reactive Circuits. 
For instance, next steps may involve integrating heterogeneous computation models into RC's formula nodes, e.g., depending on the size of the respective sub-formula.
Looking further, one may explore integrating RCs with a wider range of probabilistic and neural models and leverage more active interactions, e.g., by reading frequencies and requesting upstream models at optimized time intervals.
Overall, we believe our work paves the way for a new class of agile and efficient neuro-symbolic systems capable of real-time reasoning in complex, dynamic worlds.

\begin{acks}
Simon Kohaut gratefully acknowledges the financial support from the Honda Research Institute Europe~(HRI-EU).
The Eindhoven University of Technology authors received support from their Department of Mathematics and Computer Science and the Eindhoven Artificial Intelligence Systems Institute.
This work was supported by the Konrad Zuse School of Excellence in Learning and Intelligent Systems (ELIZA).
This work has benefited from the Cluster of Excellence "Reasonable AI" funded by the German Research Foundation (DFG) under Germany’s Excellence Strategy— EXC-3057.
All emojis used in the figures were designed by OpenMoji under the CC BY-SA 4.0 license.
\end{acks}

\printbibliography

\appendix

\clearpage
\section{Resin Grammar}
\label{app:grammar}

In the following, we provide the grammar of the Resin asynchronous probabilistic programming language in Backus-Naur form.

$$
\begin{aligned}
\text{-- Program Structure --} \\
\langle \text{program} \rangle     &::= \langle \text{statement} \rangle \mid \langle \text{statement} \rangle \langle \text{program} \rangle \\
\langle \text{statement} \rangle   &::= \langle \text{source\_decl} \rangle \mid \langle \text{logic\_rule} \rangle \mid \langle \text{target\_decl} \rangle \mid \langle \text{comment} \rangle \\
\\
\text{-- Asynchronous Datastreams --} \\
\langle \text{source\_decl} \rangle &::= \langle \text{id} \rangle \text{ "<- source(" } \langle \text{path} \rangle \text{ ", " } \langle \text{type} \rangle \text{ ")." } \\
\langle \text{target\_decl} \rangle &::= \langle \text{id} \rangle \text{ " -> target(" } \langle \text{path} \rangle \text{ ")." } \\
\langle \text{type} \rangle        &::= \text{ "Probability" } \mid \text{ "Density" } \mid \text{ "Number" } \mid \text{ "Boolean" } \\\\
\text{-- First-Order Logic --} \\
\langle \text{clause} \rangle  &::= \langle \text{head} \rangle \text{ " if " } \langle \text{body} \rangle \text{ "." } 
\\
\langle \text{head} \rangle        &::= \langle \text{atom} \rangle \\
\langle \text{body} \rangle        &::= \langle \text{literal} \rangle \mid \langle \text{body} \rangle \text{ " and " } \langle \text{literal} \rangle \\
\langle \text{literal} \rangle     &::= \langle \text{atom} \rangle \mid \text{ "not " } \langle \text{atom} \rangle \\
\langle \text{atom} \rangle        &::= \langle \text{predicate} \rangle \mid \langle \text{comparison} \rangle \\
\\
\text{-- Expressions \& Terms --} \\
\langle \text{comparison} \rangle  &::= \langle \text{term} \rangle \text{ } \langle \text{rel\_op} \rangle \text{ } \langle \text{term} \rangle \\
\langle \text{predicate} \rangle   &::= \langle \text{id} \rangle \text{ "(" } \langle \text{arg\_list} \rangle \text{ ")" } \mid \langle \text{id} \rangle \\
\langle \text{arg\_list} \rangle   &::= \langle \text{term} \rangle \mid \langle \text{term} \rangle \text{ ", " } \langle \text{arg\_list} \rangle \\
\langle \text{term} \rangle        &::= \langle \text{variable} \rangle \mid \langle \text{id} \rangle \mid \langle \text{number} \rangle \mid \langle \text{predicate} \rangle \\
\\
\text{-- Tokens --} \\
\langle \text{rel\_op} \rangle     &::= \text{ ">" } \mid \text{ "<" } \mid \text{ "==" } \mid \text{ ">=" } \mid \text{ "<=" } \\
\langle \text{variable} \rangle    &::= [A-Z][a-zA-Z0-9\_]* \\
\langle \text{id} \rangle          &::= [a-z][a-zA-Z0-9\_]* \\
\langle \text{path} \rangle        &::= \text{ "\"" } [a-zA-Z0-9\_/]* \text{ "\"" } \\
\langle \text{comment} \rangle     &::= \text{ "\# " } \text{ (any characters) }
\end{aligned}
$$

\clearpage
\section{Reactive Circuit Algorithms}
\label{app:algorithms}

In the following, we provide the \textit{lift} and \textit{drop} algorithms employed for online adaptation of Reactive Circuits.

\input{algorithms/lift}
\input{algorithms/drop}
\end{document}

%% file: header.tex
\usepackage{url}
\usepackage{amsmath,amsfonts}
\usepackage{graphicx}
\usepackage{textcomp}
\usepackage{xcolor}
\usepackage{epstopdf}
\usepackage{upgreek}
\usepackage{comment}
\usepackage{graphicx}
\usepackage{tikz}
\usepackage{lipsum}
\usepackage{subcaption}
\usepackage{svg}
\usepackage{stmaryrd}
\usepackage{algorithm}
\usepackage{algorithmicx}
\usepackage{algpseudocode}
\usepackage[frozencache,cachedir=.]{minted}
\usepackage{newfloat}
\usepackage{listings}
\DeclareCaptionStyle{ruled}{labelfont=normalfont,labelsep=colon,strut=off} 
\floatstyle{ruled}
\newfloat{listing}{tb}{lst}{}
\floatname{listing}{Listing}

\let\oldhat\hat

\renewcommand{\hat}[1]{\oldhat{\mathbf{#1}}}

\newcommand{\focn}{\lambda_{\phi, n}(t)}

%% file: figures/architecture.tex
\begin{figure*}
    \centering
    \includegraphics[width=\linewidth]{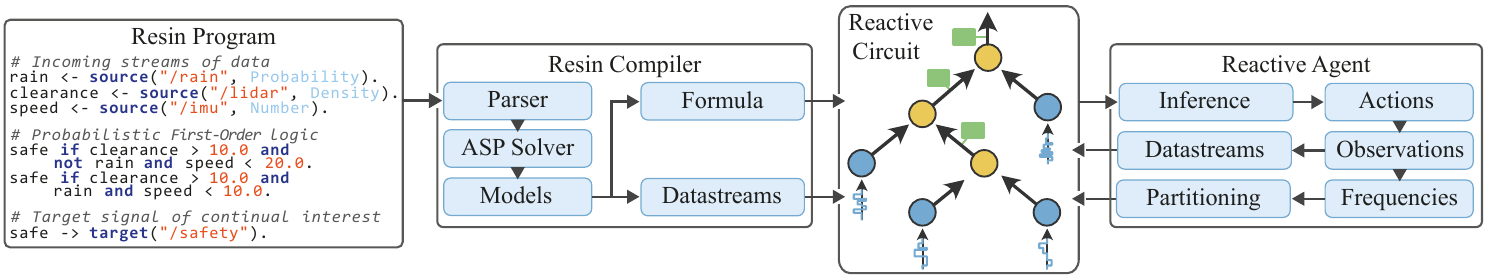}
    \caption{
        \textbf{The Asynchronous Reasoning architecture:}
        For each target of a Resin program, the compiler parses the program into a constrained Answer Set Program and generates the respective stable models.
        This causes an initial Reactive Circuit (RC) to be created and connected to datastreams, which is then, over time, adapted to reflect the volatility of its inputs.
    }
    \label{fig:architecture}%
\end{figure*}%

%% file: figures/safety_rc.tex
\begin{figure}[t]
    \begin{subfigure}{0.38\textwidth}
        \begin{minted}
            [
                frame=none,
                autogobble,
                fontsize=\footnotesize,
                xleftmargin=20pt,
                linenos
            ]{python}
            # Incoming streams of data
            rain <- source("/rain", Probability).
            clearance <- source("/lidar", Density).
            speed <- source("/imu", Number).
            
            # Probabilistic logic
            safe if clearance > 10.0 and
                not rain and speed < 20.0.
            safe if clearance > 10.0 and
                rain and speed < 10.0.
            
            # Target signal of continual interest
            safe -> target("/safety").
        \end{minted}
    \caption{
        A Resin program
    }%
    \end{subfigure}
    \hfill
    \begin{subfigure}{0.2\textwidth}
        \includegraphics[width=\textwidth]{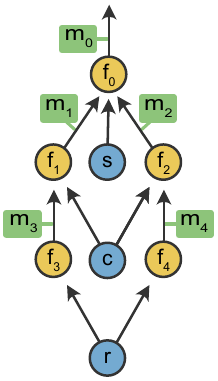}
        \centering
        \caption{The Reactive Circuit}
    \end{subfigure}
    \hfill
    \begin{subfigure}{0.3\textwidth}
        \includegraphics[width=\textwidth]{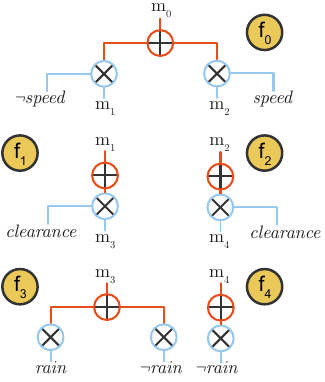}
        \centering
        \caption{The Algebraic Circuits}
    \end{subfigure}
    \caption{
        \textbf{Reactive Circuits enable Asynchronous Reasoning:}
        After compiling a Resin program (a) into an initial Reactive Circuit, its structure is adapted over time to facilitate an optimized inference scheme (b) by dividing the original formula into memorized sub-formulas according to the volatility of the respective input signals (c). 
        Here, the $speed$, $clearance$, and $rain$ signals are assumed to have a descending volatility, thereby being assigned to separate depths of the RC's DAG.
    }
    \label{fig:safety_rc}
\end{figure}%

%% file: figures/adaptation.tex
\definecolor{custom_green}{HTML}{8EC47A}
\definecolor{custom_yellow}{HTML}{E0C050}
\definecolor{custom_blue}{HTML}{74AAD1}
\begin{figure}[t]
    \centering
    \includegraphics[width=0.85\linewidth]{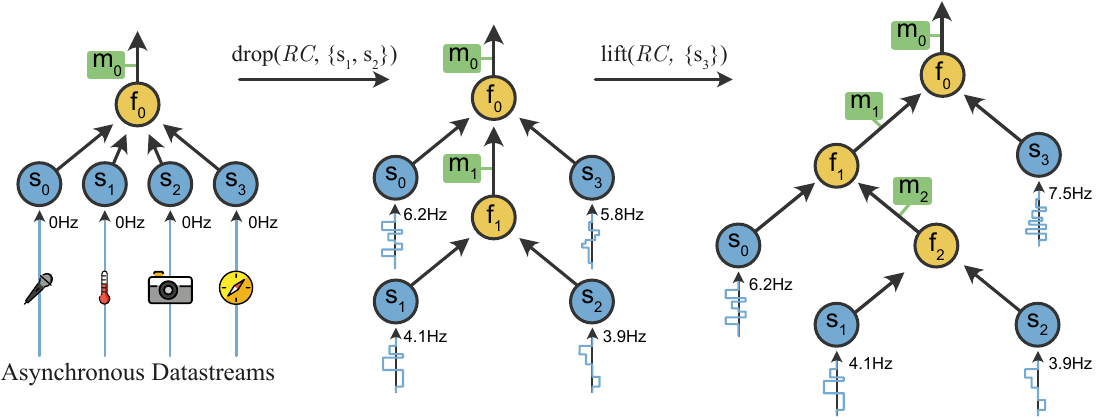}
    \caption{
    \textbf{Asynchronous Inference in Resin via Reactive Circuits:}
        For each inference target of a Resin program, a Reactive Circuit (RC) is initialized.
        Over time, the RC's structure is adapted to facilitate memorized intermediate results (\textcolor{custom_green}{green}) and offload sub-formulas of the sum-product problem (\textcolor{custom_yellow}{yellow}), enabling efficient updates based on a partitioning of the frequency of new values arriving through the source nodes (\textcolor{custom_blue}{blue}).
    }
    \label{fig:adaptation}%
\end{figure}%

%% file: figures/operations.tex

%% file: figures/time.tex
\begin{figure}[t!]
    \centering
    \begin{subfigure}{0.49\textwidth}
        \centering
        \includegraphics[width=\linewidth]{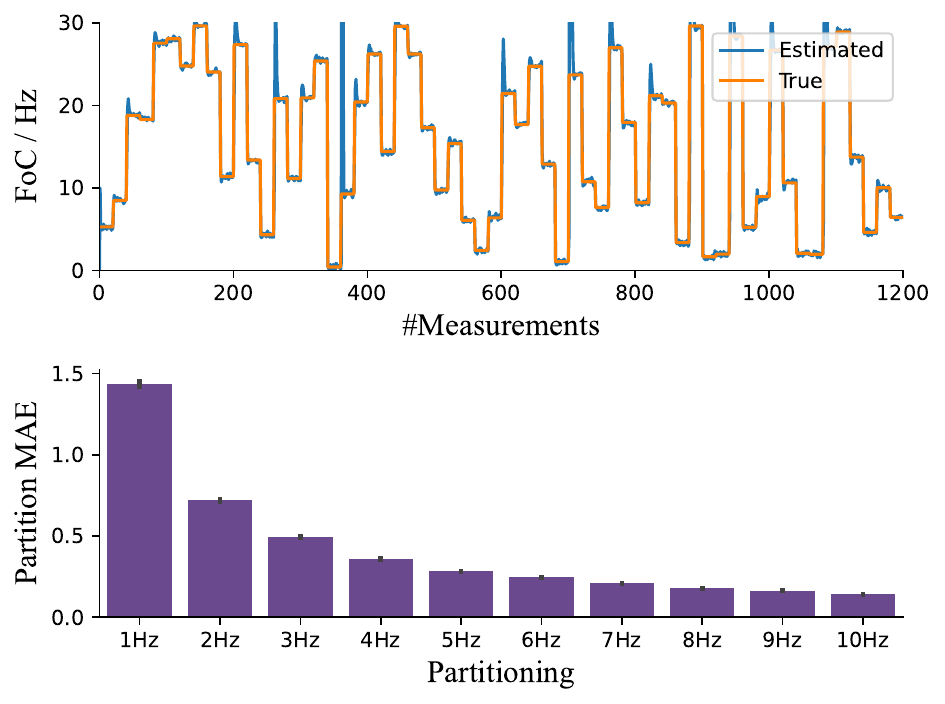}
        \caption{FoC Estimation}
        \label{fig:foc_est}
    \end{subfigure}
    \hfill
    \begin{subfigure}{0.49\textwidth}
        \centering
        \includegraphics[width=\linewidth]{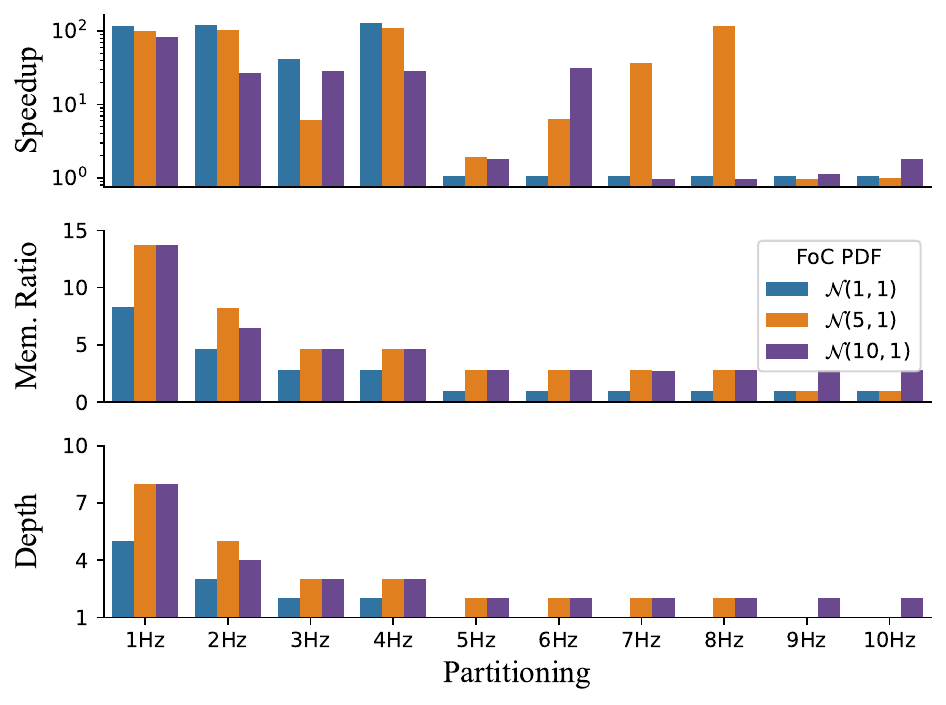}
        \caption{Reactive Circuit Adaptation}
        \label{fig:rc_adapt}
    \end{subfigure}
    
    \caption{
        \textbf{Simulation study of FoC estimation and time-dynamic Reactive Circuit properties.} 
        Evaluation on asynchronous synthetic data reveals that, while the filter can track time-variant frequencies (sampled from $U(0, 30)$ every $20$ measurements), the clustering accuracy is sensitive to the granularity of the partitioning (a). 
        As the partition size varies, the adapted RC exhibits a direct link between depth and memory footprint, whereas speedup displays non-monotonic behavior relative to a static baseline (b). 
        Notably, the partition Mean Absolute Error (MAE) is governed by partitioning granularity, and adaptations cease entirely when a single partition encompasses all signals.
    }
    \label{fig:combined_simulation}
\end{figure}

%% file: figures/simulation.tex
\begin{figure}
    \centering
    \begin{subfigure}{0.24\linewidth}
        \centering
        \includegraphics[width=\textwidth]{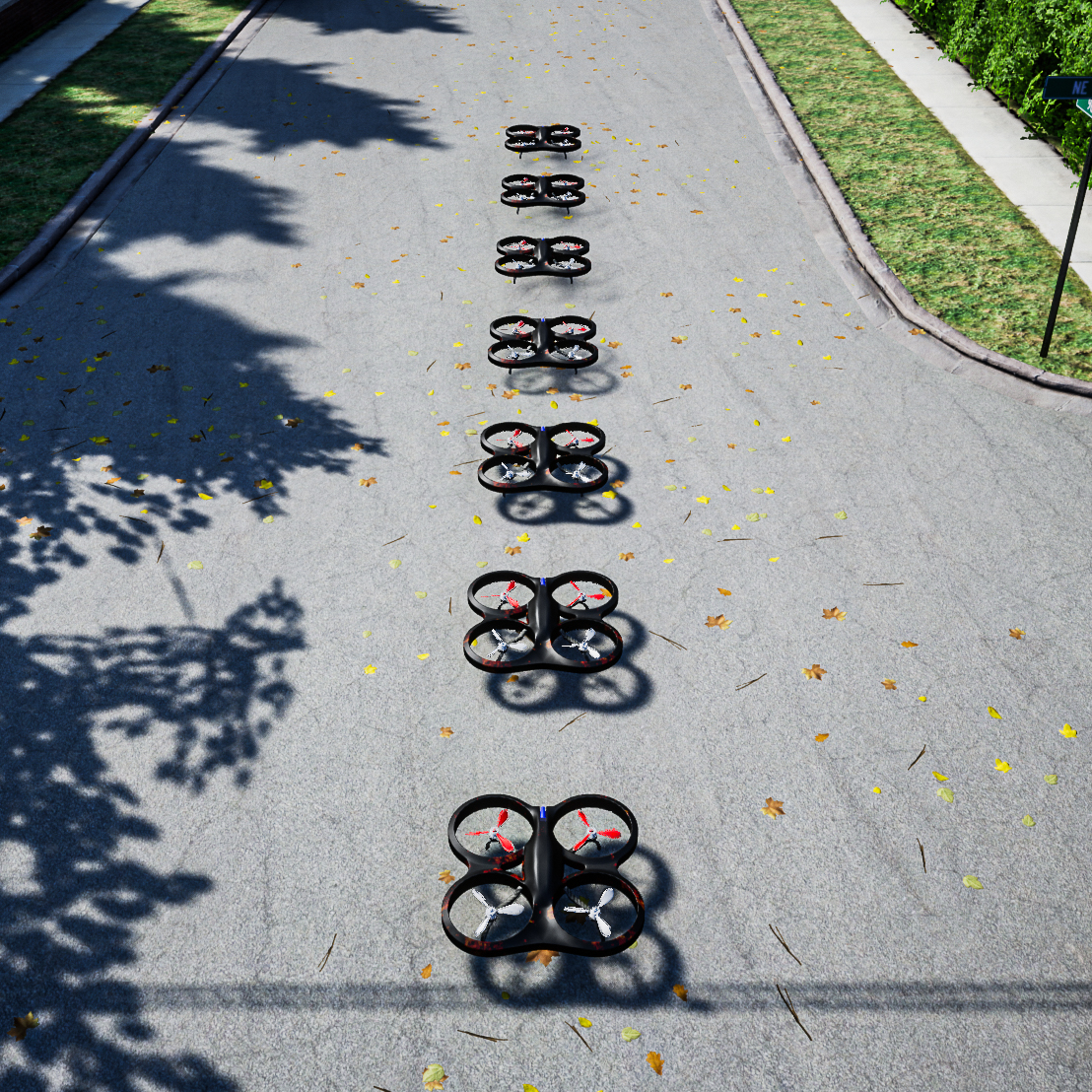}
        \caption{Parking}
    \end{subfigure}
    \hfill
    \begin{subfigure}{0.24\linewidth}
        \centering
        \includegraphics[width=\textwidth]{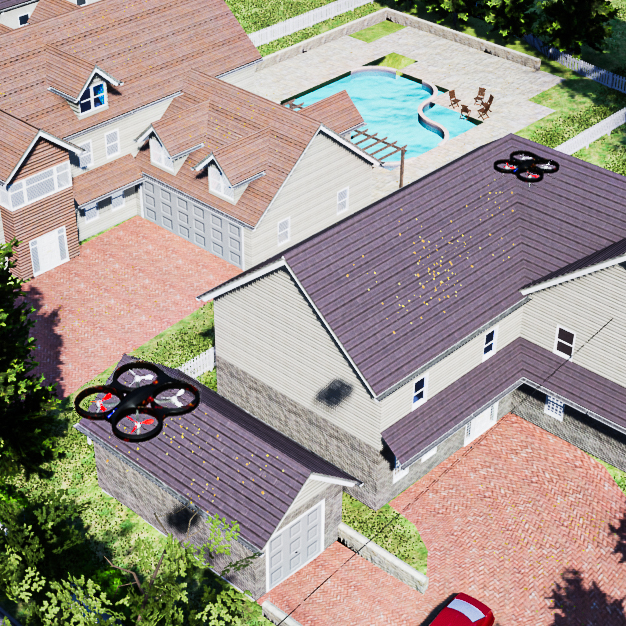}
        \caption{Travel}
    \end{subfigure}
    \hfill
    \begin{subfigure}{0.48\linewidth}
        \includegraphics[width=\textwidth]{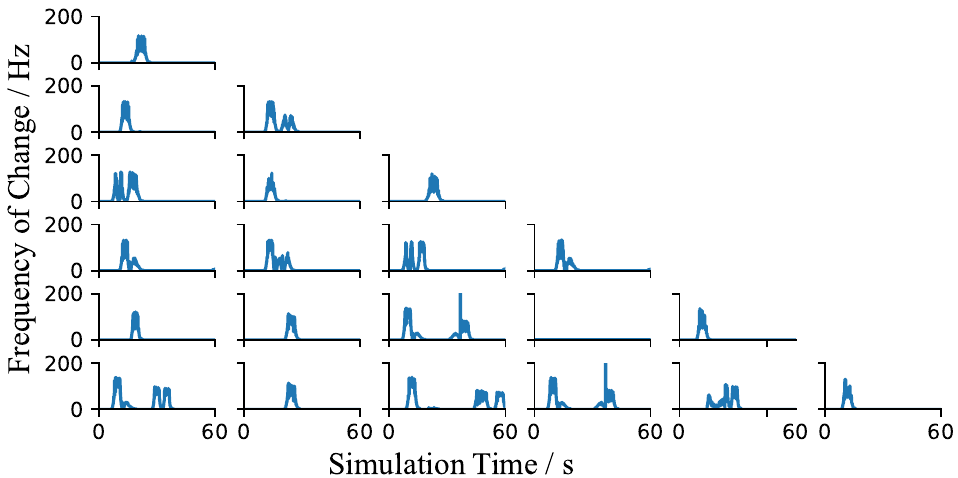}
        \caption{Frequencies}
    \end{subfigure}
    \caption{
        \textbf{Simulated Unmanned Aircraft System traffic over suburban environment:}
        Randomized journeys are triggered for each UAS, leading them to depart from and return to the common parking area (a) and to repeatedly enter each other's airspaces~(b).
        Meanwhile, the Frequencies of Change of the pairwise distances are tracked (c) via individual Kalman filter processes, showing how the volatility of the signals is time-dynamic, peaking when random journeys lead to similar areas or when returning to the parking zone.
    }
    \label{fig:simulation}
\end{figure}

%% file: listings/flight_conditions.tex
\begin{listing}[t]
\caption{
    \textbf{Continual safety assertion in Resin:}
    By estimating pairwise distances between individual UASs as illustrated in Fig.~\ref{fig:simulation}, a Resin program can continually assert the safety within the drone network.
}%
\label{listing:flight}%
    \centering
    \begin{minted}
    [
        frame=none,
        autogobble,
        fontsize=\footnotesize,
        xleftmargin=20pt,
        linenos
    ]{python}
    # Estimated pairwise distances
    distance(drone_1, drone_2) <- source(
        "/drone1_drone2", Density
    ).
    distance(drone_1, drone_3) <- source(
        "/drone1_drone3", Density
    ).
    ...
    
    # The situation becomes unsafe if drones enter each other's airspace,
    # i.e., closer than 25 meters
    unsafe if distance(X, Y) < 25. 

    # We continually provide the `unsafe` atom as target signal via a single RC
    unsafe -> target("/safety").
    \end{minted}
\end{listing}%

%% file: figures/airsim_time.tex
\begin{figure}[t!]
    \centering
    \begin{subfigure}{0.48\textwidth}
        \centering
        \includegraphics[width=\linewidth]{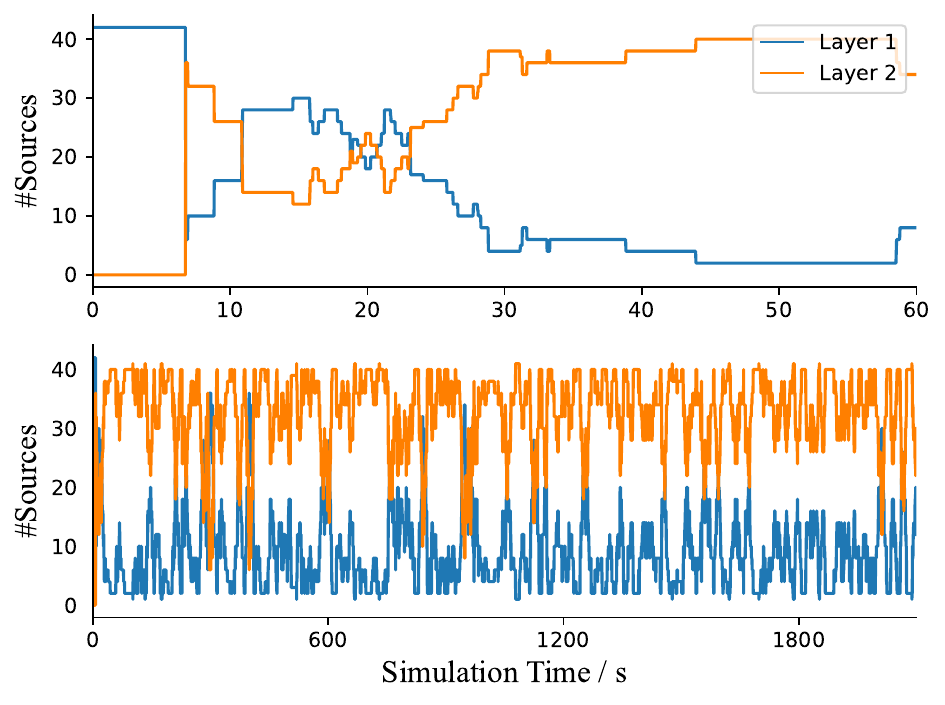}
        \caption{Frequency adaptation over UAS traffic} 
        \label{fig:leafs}
    \end{subfigure}
    \hfill
    \begin{subfigure}{0.48\textwidth}
        \centering
        \includegraphics[width=\linewidth]{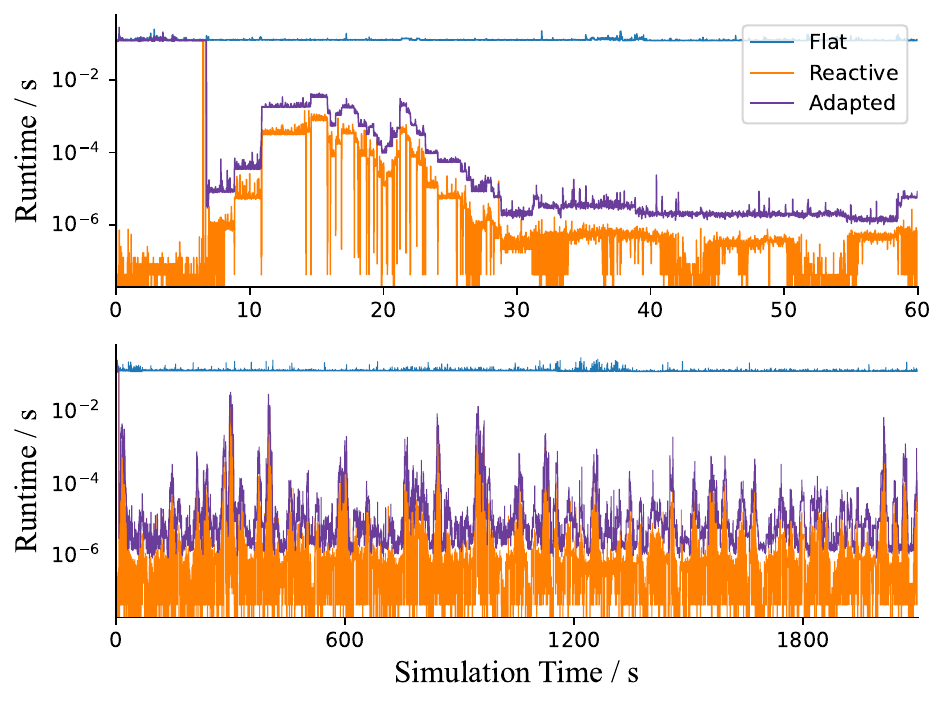}
        \caption{Reactive safety assertions} 
        \label{fig:airsim_time}
    \end{subfigure}
    
    \caption{
        \textbf{Ablation study of frequency adaptation and performance in UAS traffic scenarios.} 
        As a drone swarm performs randomized journeys, the RC redistributes sources based on the estimated FoCs. 
        A two-layer partitioning effectively captures the distinction between static and dynamic sources (a), where the signal count in the first layer directly correlates with the system runtime. 
        This adaptive partitioning facilitates significant performance gains (b); while the adapted circuit achieves high speedups through knowledge compilation alone, the reactive execution via memorization of frequency-separated formula slices yields multiple orders of magnitude in speedup, as shown on the logarithmic scale.
    }
    \label{fig:airsim_ablation}
\end{figure}

%% file: algorithms/lift.tex
\begin{algorithm}
\caption{\textbf{Lift} applied to Reactive Circuit $\mathcal{RC}$ and set of signals $\mathcal{S}_l$}
\label{alg:lift}
\begin{algorithmic}[1]
\For {$s \in$ $\mathcal{S}_l$} \Comment{For each signal to lift}
    \For {$f_i \in \mathcal{F}_\mathcal{RC}$} \Comment{For each formula in the RC}
        \If {$s \in children(f_i)$} \Comment{Move dependencies on $s$ to parent nodes}
            \If {$parents(f_i)$ is empty} \Comment{Create a new parent if none exists}
                \State Add parent node $f_j$ with formula $f_j(s, m_i) = s \cdot m_i$ to $\mathcal{F}_\mathcal{RC}$
                \State Add edges $(s, f_j)$ and $(f_i, f_j)$ to $\mathcal{E}_\mathcal{RC}$
                \State $m_j \leftarrow m_i$
            \Else
                \For {$f_j \in parents(f_i)$} \Comment{Else, multiply all parents with $s$}
                    \State Add edge $(s, f_j)$ to $\mathcal{E}_\mathcal{RC}$ and multiply $m_i$ in formula $f_j$ with $s$
                \EndFor
            \EndIf
            \State Remove edge $(s, f_i)$ from $\mathcal{E}_\mathcal{RC}$ and $s$ from formula $f_i$
            \State $m_i \leftarrow m_i \mathbin{/} s$
        \EndIf
    \EndFor
\EndFor
\end{algorithmic}
\end{algorithm}

%% file: algorithms/drop.tex
\begin{algorithm}
\caption{\textbf{Drop} applied on Reactive Circuit $\mathcal{RC}$ and set of signals $\mathcal{S}_d$}
\label{alg:drop}
\begin{algorithmic}[1]
\For {$s \in$ $\mathcal{S}_d$} \Comment{For each signal to drop}
    \For {$f_i \in \mathcal{F}_\mathcal{RC}$} \Comment{For each formula in the RC}
        \If {$s \in children(f_i)$} \Comment{Move dependencies on $s$ to child nodes}
            \If {$children(f_i) \setminus s$ is empty} \Comment{Create a new child if none exists}
                \For {each $\otimes$ node in $f_i$ that connects to $s$}
                    \State Add child node $f_j$ with formula $f_j(s) = s$ to $\mathcal{F}_\mathcal{RC}$
                    \State Add edges $(s, f_j)$ and $(f_j, f_i)$ to $\mathcal{E}_\mathcal{RC}$
                    \State $m_j \leftarrow s$
                    \State Substitute respective $s$ in formula $f_i$ with $m_j$
                \EndFor
            \Else
                \For {$f_j \in children(f_i)$} \Comment{Else, multiply all child formulas with $s$}
                    \State Add edge $(s, f_j)$ to $\mathcal{E}_\mathcal{RC}$ and multiply formula $f_i$ with $s$
                \EndFor
            \EndIf
        \State Remove edge $(s, f_i)$ from $\mathcal{E}_\mathcal{RC}$ and $s$ from formula $f_i$
        \EndIf
    \EndFor
\EndFor
\end{algorithmic}
\end{algorithm}